\documentclass[11pt]{article}
\usepackage{amsthm,amsmath, amssymb, latexsym}
\usepackage{wrapfig,lipsum}
\usepackage{algorithm, algorithmic}
\usepackage{enumitem}
\usepackage{graphicx}
\usepackage{xcolor}
\usepackage{fancyhdr}
\usepackage{yfonts,mathrsfs,skak}
\usepackage[mathcal]{euscript}
\usepackage{rotating}
\usepackage[colorlinks=true,linkcolor=blue,citecolor=violet]{hyperref}
\usepackage{stackengine}
\usepackage[framemethod=tikz]{mdframed}
\usepackage{MnSymbol}

\newcommand{\muca}{\mu_{\scalebox{0.6}{\ensuremath{\mathcal A}}}}
\newcommand{\Constr}{\mathcal{C}}
\newcommand{\bw}{\boldsymbol{w}}
\newcommand{\bx}{\boldsymbol{x}}
\newcommand{\by}{\boldsymbol{y}}
\newcommand{\bz}{\boldsymbol{z}}

\newcommand{\dham}{\text{d}_{\scalebox{0.6}{H}}}

\newcommand{\mbf}[1]{\mathbf{#1}}

\newcommand{\mrm}[1]{\mathrm{#1}}

\newcommand{\pred}{\widehat{{y}}}

\newcommand{\argmin}[1]{\underset{#1}{\mrm{argmin}} \ }

\newcommand{\reals}{\mathbb{R}}
\newcommand{\En}{\mathbb{E}}  

\newcommand{\inner}[1]{\left\langle #1 \right\rangle}

\newcommand{\ind}[1]{{\bf{1}}\left\{#1\right\}}
\newcommand{\tr}{\ensuremath{{\scriptscriptstyle\mathsf{T}}}}

\newcommand{\bepsilon}{\boldsymbol{\varepsilon}}

\newcommand\x{\mathbf{x}}

\newcommand\cA{\mathcal{A}}

\newcommand\cD{\mathcal{D}}

\newcommand\X{\mathcal{X}}

\newcommand\F{\mathcal{F}}

\newcommand\Rad{\mathscr{R}}

\newcommand{\Rel}{\mbf{Rel}}

\def\deq{\triangleq}

\newcommand{\OPT}{{\sf OPT}}
\newcommand{\SDP}{{\sf SDP}}

\theoremstyle{plain}
\newtheorem{theorem}{Theorem}
\newtheorem{lemma}[theorem]{Lemma}

\theoremstyle{definition}
\newtheorem{example}{Example}

\newtheorem{observation}{Observation}
\theoremstyle{definition}

\usepackage{bbm}

\title{A Tutorial on Online Supervised Learning with Applications to Node Classification in Social Networks}
\author{Alexander Rakhlin \\ University of Pennsylvania \and Karthik Sridharan \\ Cornell University}

\date{\today}
\makeindex

\begin{document}
\maketitle

We revisit the elegant observation of T. Cover \cite{Cover65behaviour} which, perhaps, is not as well-known to the broader community as it should be. The first goal of the tutorial is to explain---through the prism of this elementary result---how to solve certain sequence prediction problems by modeling sets of solutions rather than the unknown data-generating mechanism. We extend Cover's observation in several directions and focus on  computational aspects of the proposed algorithms. The applicability of the methods is illustrated on several examples, including node classification in a network. 

The second aim of this tutorial is to demonstrate the following phenomenon: it is possible to predict as well as a combinatorial ``benchmark'' for which we have a certain multiplicative approximation algorithm, even if the exact computation of the benchmark given all the data is NP-hard. The proposed prediction methods, therefore, circumvent some of the computational difficulties associated with finding the best model given the data. These difficulties arise rather quickly when one attempts to develop a probabilistic model for graph-based or other problems with a combinatorial structure.

\section{The basics of bit prediction}
\label{sec:basics}

Consider the task of predicting an unknown sequence $\by=(y_1,\ldots,y_n)$ of $\pm1$'s in a streaming fashion. At time $t=1,\ldots,n$, a forecaster chooses $\pred_t\in\{\pm1\}$ based on the history $y_1,\ldots,y_{t-1}$  observed so far. After this prediction is made, the value $y_t$ is revealed to the forecaster. The average number of mistakes incurred on the sequence is
\begin{align}
	\frac{1}{n}\sum_{t=1}^n \ind{\pred_t\neq y_t},
\end{align}
where $\ind{S}$ is $1$ if $S$ is true, and $0$ otherwise. A randomized algorithm $\cA$ is determined by the means
\begin{align}
	\widehat{q}_t = \widehat{q}_t (y_1,\ldots,y_{t-1}) \in [-1,1], ~~~~ t=1,\ldots,n
\end{align}
of the distributions $\cA$ puts on the outcomes $\{\pm1\}$ at time $t$. The expected average number of mistakes made on the sequence $\by$ by a randomized algorithm $\cA$ is 
\begin{align}
	\label{eq:avg_error}
	\muca(\by) = \En\left[ \frac{1}{n}\sum_{t=1}^n \ind{\pred_t\neq y_t} \right],
\end{align}
where the expectation is with respect to the random choices $\pred_t$, drawn from the  distributions with means $\widehat{q}_t(y_1,\ldots,y_{t-1})$, $t=1,\ldots,n$. 

Whenever a prediction algorithm $\cA$ has low expected error on some sequence $\by$, it must be at the expense of being worse on other sequences. Why? On average over the $2^n$ sequences, the algorithm necessarily incurs an error of $1/2$. Indeed, denoting by $\bepsilon = (\varepsilon_1,\ldots,\varepsilon_n)$ a sequence of independent unbiased $\pm1$-valued (Rademacher) random variables, it holds that
\begin{align}
	\label{eq:avg_cost_over_all_seq}
	\frac{1}{2^n}\sum_{\by} \muca(\by) = \En_{\bepsilon}\En\left[ \frac{1}{n}\sum_{t=1}^n  \ind{\pred_t\neq \varepsilon_t} \right] = \frac{1}{2}
\end{align}
by an elementary inductive calculation, keeping in mind that $\widehat{q}_t=\widehat{q}_t(\varepsilon_1,\ldots,\varepsilon_{t-1})$. As a consequence, it is impossible to compare prediction algorithms when all sequences are treated equally.

Evidently, any algorithm $\cA$ induces a function $\muca$ on the hypercube $\{\pm1\}^n$, whose average value is $1/2$. Cover \cite{Cover65behaviour} asked the converse: given a function $\phi:\{\pm1\}^n\to \reals$, is there an algorithm $\cA$ with the property
\begin{align}
	\label{eq:achievable}
	\forall \by,~~~ \muca(\by) = \phi(\by).
\end{align}
In words, if we specify the average number of mistakes we are willing to tolerate for each sequence, is there an algorithm that achieves the goal?
If such an algorithm exists, we shall say that $\phi$ is \emph{achievable}. Let us call $\phi$ \emph{stable} if
\begin{align}
	\label{eq:stab}
	|\phi(\ldots,+1,\ldots)-\phi(\ldots,-1,\ldots)|\leq \frac{1}{n}
\end{align}
for any coordinate, keeping the rest fixed. Cover's observation \cite{Cover65behaviour} is now summarized as
\begin{lemma}
	\label{lem:cover}
	Suppose $\phi:\{\pm1\}^n\to\reals$ is stable. Then 
	\begin{center}
		$\phi$ is achievable if and only if $\En\phi = 1/2$,
	\end{center} where the expectation is under the \emph{uniform} distribution.
\end{lemma}
That is, for any function $\phi$ that does not change too fast along any edge of the hypercube, there exists an algorithm that attains the average number of mistakes given by $\phi$ if and only if $\phi$ is $1/2$ on average over all $2^n$ sequences. 

As an immediate consequence, for any stable $\phi$,  $\En\phi\geq 1/2$ is equivalent to existence of an algorithm with $$\forall \by,~~~ \muca(\by)\leq \phi(\by).$$ This latter version of the Lemma  will be used in the sequel, and we shall say that $\phi$ is achievable even if \eqref{eq:achievable} holds with an inequality.

Perhaps, it is worth emphasizing the following message of the lemma:
\begin{quote}
	\sf
	Existence of a forecasting strategy with a given mistake bound for an arbitrary sequence can be verified by checking a probabilistic inequality.
\end{quote}
The proof of the more general multi-class statement (Lemma~\ref{lem:cover_multi}) appears in the appendix; it  uses backward induction, and may be viewed as a ``potential function'' argument. 

\begin{example}
	\label{ex:1}
	Let $\bar{\by}=\frac{1}{n}\sum_{t=1}^n \ind{y_i=1}$ denote the proportion of $+1$'s in the sequence. Take $$\phi(\by)=\min\{\bar{\by}, 1-\bar{\by}\} + Cn^{-1/2}.$$
	It is an exercise to show that the mean of this function with respect to the uniform distribution is at least $1/2$ for an appropriate constant $C$, and that $\phi$ is stable. Hence, \emph{there exists} a randomized prediction algorithm $\cA$ that takes advantage of imbalanced sequences, in the sense that
\begin{align}
	\label{eq:imbalanced}
	\forall \by,~~~ \muca(\by)\leq \min\{\bar{\by}, 1-\bar{\by}\} + Cn^{-1/2}.
\end{align}
For instance, if the sequence $\by$ ends up having $30\%$ of $1$'s, the algorithm, in expectation, will incur roughly $30\%$ proportion of errors (the extra term $Cn^{-1/2}$ is small for large enough $n$). Notably, the mistake guarantee holds for \emph{any} sequence without any stochasticity assumption on its nature\footnote{In \cite{Blackwell95}, D. Blackwell draws parallels between an almost sure version of \eqref{eq:imbalanced} (based on ``approachability'') and the corresponding i.i.d. statement.}
 ~and the imbalance of the sequence need not be known until the very end of the $n$ rounds. The existence of such a prediction strategy may seem rather surprising, and for the intrigued reader that attempts to solve this problem, let us give a hint: no deterministic method will work.
\end{example}

\section{Modeling solutions through $\phi$}

As we have seen, there is no algorithm that can predict sequences uniformly better than another algorithm. Thankfully, we do not care about \emph{all} sequences. A typical prediction problem has some structure that informs us of the sequences we should focus on. The structure is often captured through a stochastic description of the generative process, such as an i.i.d. or an  autoregressive assumption. The stochastic assumption, however, may not be justified in applications that involve complex interactions and time dependencies, such as in the social network example below. 

The approach in this tutorial is different: we provide a non-stochastic description of the ``prior knowledge'' via the function $\phi$. The function specifies the expected proportion of mistakes we are willing to tolerate on each sequence. We tilt $\phi$ down towards the sequences we care about, at the expense of making it larger on some other sequences that we are unlikely to see anyway. Lemma~\ref{lem:cover} guarantees existence of a prediction strategy with proportion of mistakes given by $\phi$, as long as $\phi$ is stable and at least $1/2$ on average. Furthermore, given $\phi$, the algorithm is simple to state, as we will see below.

In 1950's, David Hagelbarger \cite{hagelbarger1956seer} and Claude Shannon \cite{shannon1953mind} at the Bell Labs built the so-called ``mind reading machines'' to play the game of matching pennies. According to some accounts,\footnote{\url{http://backup.itsoc.org/review/meir/node1.html}} the machine was consistently predicting the sequence of bits produced by untrained human players better than 50\%.\footnote{Here is a modern version of this machine: \url{http://www.mindreaderpro.appspot.com} by Y. Freund and colleagues.
} Of course, the only reason a machine can predict consistently better than chance is that humans cannot enter ``truly random'' sequences. 

How can we design such a machine? Using the approach outlined above, we would need to capture the possible patterns of behavior we might see in the sequences and encode this knowledge in $\phi$. We have already seen in Example~\ref{ex:1} how to take advantage of imbalanced sequences. Of course, this may not be the only structure available, and we shall now describe a few general approaches to building $\phi$.

The first basic construction will be called \emph{aggregation}. Suppose $\phi_1,\ldots,\phi_N$ are $N$ stable functions, each satisfying $\En\phi_i\geq 1/2$. It is then possible to show that the best-of-all aggregate
\begin{align}
	\label{eq:min_exp}
	\phi(\by) = \min_{j\in\{1,\ldots,N\}} \phi_j(\by) + C_{n,N}
\end{align}
is stable and satisfies $\En\phi\geq 1/2$ for  
$$C_{n,N} = \sqrt{\frac{c\log N}{n}},$$
with an absolute constant $c$. The penalty $C_{n,N}$ for aggregating $N$ ``experts'' depends only logarithmically on $N$, and diminishes when $n$ is large. A reader familiar with the literature on prediction with expert advice will recognize the form of $C_{n,N}$ as a \emph{regret bound} of the Exponential Weights algorithm.\footnote{In contrast to an algorithmic proof of the regret bound, we derived $C_{n,N}$ as a value necessary to ensure $\En\phi\geq 1/2$.}

Another useful (and most-studied) way to construct $\phi$ is by taking a subset $F\subseteq \{\pm1\}^n$ and letting
\begin{align}
	\label{eq:def_phi_dham}
	\phi(\by) = \dham(\by,F) + C_{n,F},
\end{align}
the normalized Hamming distance between $\by$ and the set $F$, penalized by $C_{n,F}$. Recall that the normalized Hamming distance is
$$ \dham(\by,F) \deq \min_{\bw\in F} \frac{1}{n}\sum_{t=1}^n \ind{y_t\neq w_t},$$
where $\bw=(w_1,\ldots,w_n)$. The definition in \eqref{eq:def_phi_dham} automatically ensures stability of $\phi$, and the smallest $C_{n,F}$ that guarantees $\En\phi\geq 1/2$ is
$$C_{n,F} = \frac{1}{2n}\En\max_{\bw\in F} \inner{\bepsilon,\bw} \deq \Rad(F),$$
the \emph{Rademacher averages} of the set $F$.\footnote{We include the factor $1/2$ in the definition of Rademacher averages.} By Lemma~\ref{lem:cover}, there is a randomized prediction algorithm that incurs $C_{n,F}$ proportion of mistakes on any sequence in $F$, and the performance degrades linearly with the distance to the set. 

Observe that the $\phi$ function in Example~\ref{ex:1} can be written as \eqref{eq:def_phi_dham} with $F=\{-\boldsymbol{1},\boldsymbol{1}\}$. This is the simplest nontrivial set $F$, since matching the performance of a singleton $F=\{\bw\}$ simply amounts to outputting this very sequence.

As one makes $F$ a larger subset of the hypercube, the Hamming distance from any $\by$ decreases, yet the overall penalty $C_{n,F}$ becomes larger. On the extreme of this spectrum is $F=\{\pm1\}^n$. Insisting on a small error on this set is not possible, and, indeed,  $$\Rad(\{\pm1\}^n)=\frac{1}{2n}\En\max_{\bw\in \{\pm1\}^n} \inner{\bepsilon,\bw} = \frac{1}{2},$$ 
the performance of random guessing.

The goal is now clear: for the problem at hand, we would like to define $F$ to be large enough to capture the underlying structure of solutions, yet not too large. In Section~\ref{sec:computation} we come back to this issue when discussing combinatorial relaxations of $F$.

\section{Application: node classification}
\label{sec:node}

\begin{wrapfigure}{R}{0.4\textwidth}
	\vspace{-10pt}
	\begin{center}
		\includegraphics[width=.42\textwidth]{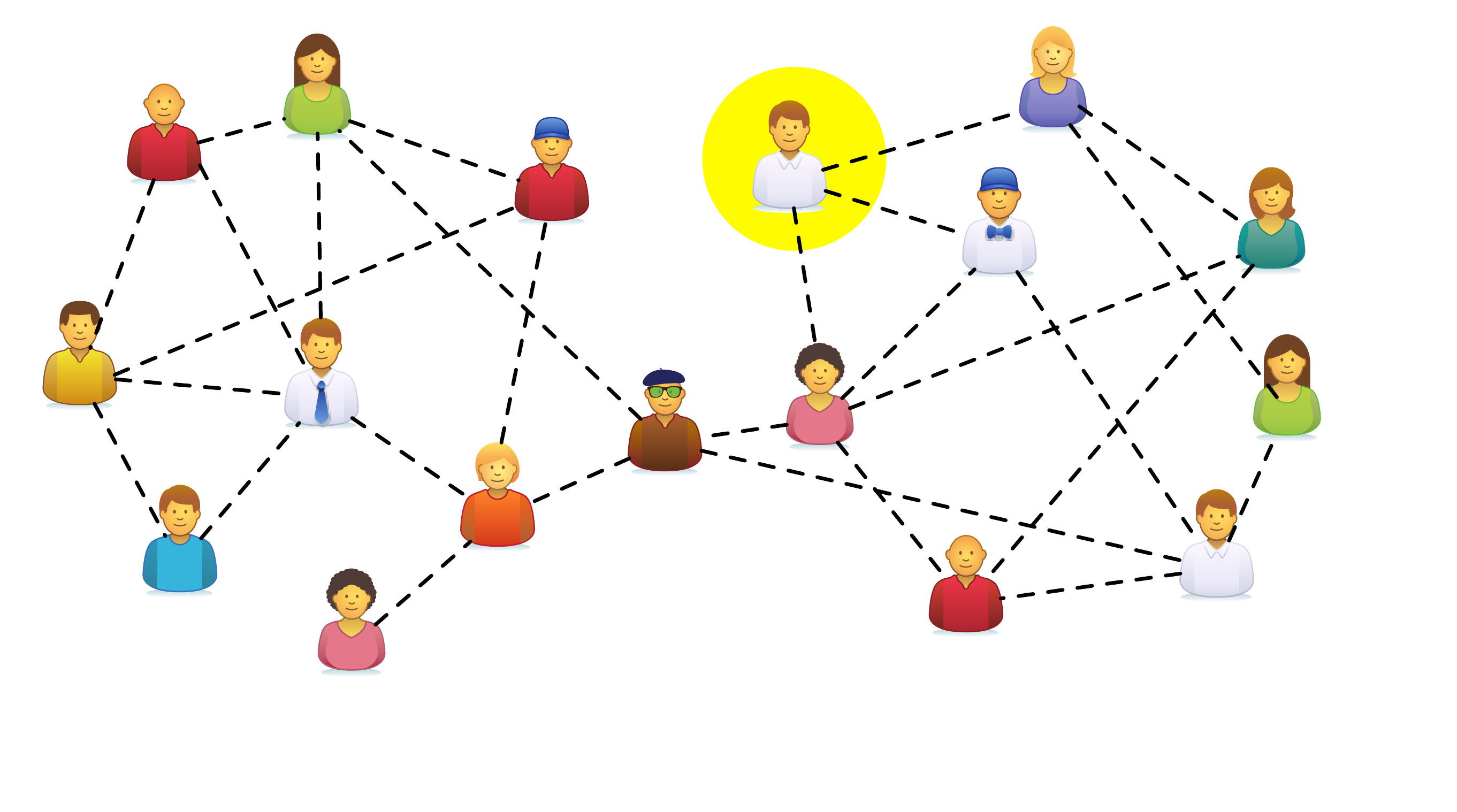}
	\end{center}
	\vspace{-30pt}
	\caption{Two-community structure.}
	\label{fig:community}
\end{wrapfigure}
We now discuss an application of Lemma~\ref{lem:cover}. Let 
$G=(V,E)$ be a known undirected graph representing a social network. At each time step $t$, a user in the network opens her Facebook page, and the system needs to decide whether to classify the user as type ``$-1$'' or ``$+1$'', say, in order to decide on an advertisement to display. We assume here that the feedback on the ``correct'' type is revealed to the system after the prediction is made. The more natural \emph{partial information} version of this problem is outside the scope of this short tutorial, and we refer the reader to \cite{rakhlin2016bistro}. 

The prediction should be made based on all the information revealed thus far (the types of users in the network that appeared before time $t$), the global graph structure, and the position of the current user in the network. In Section~\ref{sec:cov} we will also discuss the version of this problem where covariate information about the users is revealed, but at the moment assume that the graph itself provides enough signal to form a prediction. A fascinating question is: what types of $\phi$ functions capture the graph structure relevant to the problem? Below we provide two examples, only scratching the surface of what is possible. 

\subsection{Community structure}
\label{sec:community}

Suppose we have a hunch that the type of the user ($+1$ or $-1$) is correlated with the community to which she belongs. For simplicity, suppose there are two communities, more densely connected within than across (see Figure~\ref{fig:community}). To capture the idea of correlating communities and labels, we set $\phi$ to be small on labelings that assign homogenous values within each community. We make the following simplifying assumptions:  (i) $|V|=n$,  (ii) we only predict the label of each node once, and (iii) the order in which the nodes are presented is fixed (this assumption is easily removed).  Smoothness of a labeling $\by\in\{\pm1\}^n$ with respect to the graph may be computed via
\begin{align}
	\label{eq:lap1}
	\sum_{(u,v)\in E} \ind{\by(u)\neq \by(v)} = \frac{1}{4}\sum_{(u,v)\in E} (\by(u)-\by(v))^2
\end{align}
where $\by(v)\in\{\pm1\}$ is the label in $\by$ that corresponds to vertex $v\in V$. This function in \eqref{eq:lap1} counts the number of disagreements in labels at the endpoints of each edge. The value is also known as the size of the cut induced by $\by$ (the smallest possible being {\sf MinCut}). As desired, the function in \eqref{eq:lap1} gives a smaller value to the labelings that are homogenous within the communities. A more concise way to write \eqref{eq:lap1} is in terms of the graph Laplacian
\begin{align}
	\label{eq:lap2}
	\by^\tr L \by,
\end{align}
where $L = D-A$, the diagonal matrix $D$ contains degrees of the nodes, and $A$ is the adjacency matrix. 

Unfortunately, the function in \eqref{eq:lap2} is not stable. It also has an undesirable property, illustrated by the following example. The cut size is $n-1$ for a star graph, where $n-1$ nodes, labeled as $+1$, are connected to the center node, labeled as $-1$. The large value of the cut does not capture the simplicity of this labeling, which is only one bit away from being a constant $+1$.

Instead, we opt for the indirect definition \eqref{eq:def_phi_dham}. More precisely, we define
\begin{align}
	\label{eq:def_F_laplacian}
	F_\kappa = \left\{ \by\in \{\pm1\}^n ~:~ \by^\tr L \by \leq \kappa \right\}
\end{align}
for $\kappa\geq 0$, and then set
\begin{align}
	\label{eq:def_phi_Fkappa}
	\phi(\by) = \dham(\by,F_\kappa) + C_{n,F_\kappa}.
\end{align}
Parameter $\kappa$ should be larger than the value of {\sf MinCut}, for otherwise the set $F_\kappa$ is empty. The function $\phi$ has the interpretation as the proportion of vertices whose labels need to be flipped to achieve the value at most $\kappa$ for the cut, compensated by the Rademacher averages of the set $F_\kappa$. While we can give some straightforward bounds on the Rademacher averages of $F_\kappa$, the investigation of this value for various graphs, including random ones, is an interesting research question.

While {\sf MinCut} is computationally easy, the calculation of $\phi$ becomes NP-hard in general if we allow $[-1,1]$-valued weights $w_{(u,v)}$ on the edges and define $F_\kappa$ with respect to the \emph{weighted} Laplacian
\begin{align}
	\label{eq:lap2w}
	\sum_{(u,v)\in E} w_{(u,v)}(\by(u)-\by(v))^2.
\end{align}
Such a definition can be used to model trust-distrust networks, and we certainly would like to develop computationally efficient methods for this problem. Somewhat surprisingly, this is possible in certain cases, even though evaluating $\phi$ is computationally hard. See Sections~\ref{sec:rel} and \ref{sec:csp} for details.
 
\subsection{Exercise: predicting voting preferences}

Suppose $n$ individuals (connected via the known social network as in the previous example) arrive to the voting station one by one, and we are predicting whether they will vote for Grump or for Blinton. After our prediction is made, the voter reveals her true binary preference. Suppose we know the individual's place in the network and the voting preferences of the individuals observed thus far. Our task is to design an online prediction algorithm that makes as few mistakes as possible.

Suppose we have prior knowledge that Grump supporters may be  described well by a ball in the network (a ball with center $v$ and radius $r$ is the set of vertices at most $r$ hops away from $v$), but the center of this ball in the network is not known. Suppose each individual has at most $d$ friends. We leave it as an exercise to design a $\phi$ function for this prediction problem.

\section{Extension to multi-class prediction}

We now extend the result of Cover to $k$-ary outcomes, i.e. $y_t \in \{1,\ldots,k\}$. As before, the expected prediction error is given by
$$\En\left[ \frac{1}{n}\sum_{t=1}^n \ind{\pred_t\neq y_t} \right],$$
but uniformly random guessing now incurs an expected cost of $1-1/k$. By the same token, on average over the $k^n$ sequences, any algorithm must incur the expected cost of $1-1/k$. 

Let us define a couple of shorthands. We shall use the notation $a_{1:t} \deq \{a_1,\ldots,a_t\}$, $[k]\deq \{1,\ldots,k\}$,  and denote the set of probability distributions on $k$ outcomes  by $\Delta_k$.

We shall say that $\phi:[k]^n\to \reals$ is stable if for any coordinate (and holding the rest fixed),
\begin{align}
	\label{eq:stab_multiclass}
	\max_{r\in[k]}\phi(\ldots,r,\ldots) - \frac{1}{k}\sum_{i=1}^k \phi(\ldots,i,\ldots) \leq \frac{1}{nk}.
\end{align}
We now overload the notation and define a randomized forecasting strategy as a collection of distribution-valued functions of histories: $$\widehat{q}_t=\widehat{q}_t(y_{1:t-1})\in \Delta_k.$$

\begin{lemma}
	\label{lem:cover_multi}
	Suppose $\phi:[k]^n\to\reals$ is stable. Then 
	\begin{center}
		$\phi$ is achievable if and only if $\En\phi = 1-\frac{1}{k}$,
	\end{center} where the expectation is under the \emph{uniform} distribution on $[k]^n$.	
\end{lemma}
Once again, it follows from the Lemma that $\En\phi\geq 1-\frac{1}{k}$ is equivalent to  existence of a strategy with $\muca \leq \phi$. Lemma~\ref{lem:cover} can be seen as a special case of Lemma~\ref{lem:cover_multi} for $k=2$.

\section{Computation}
\label{sec:computation}

By repeatedly referring to ``existence of a prediction strategy'' in the previous sections, we, perhaps, gave the impression that these methods are difficult to find. Thankfully, this is not the case.

\subsection{The exact algorithm}
The proofs of Lemma~\ref{lem:cover} and \ref{lem:cover_multi} are constructive, and the algorithms are easy to state. For binary prediction ($k=2$), the randomized algorithm is defined on round $t$ by the mean
	\begin{eqnarray}
		\label{eq:algo_2}
		q_t^*(y_{1:t-1}) = n\cdot \En[\phi(y_{1:t-1},-1,\varepsilon_{t+1:n})- \phi(y_{1:t-1},+1,\varepsilon_{t+1:n})]
	\end{eqnarray}
	of the distribution on the outcomes $\{\pm1\}$, where the expectation is over independent Rademacher random variables $\varepsilon_{t+1},\ldots,\varepsilon_n$. The prediction $\pred_t\in\{\pm1\}$ is then a random draw such that $\En\pred_t = q_t^*$. 
	
	The two evaluations of $\phi$ in \eqref{eq:algo_2} are performed at neighboring vertices of the hypercube differing in the $t$-th coordinate. If the function values are equal in expectation, the mean $q_t^*$ is equal to zero, which corresponds to the uniform distribution on $\{\pm1\}$. In this case, the function $\phi$ does not provide any guidance on which prediction to prefer. On the other hand, if the absolute difference is $1/n$ in expectation (the largest allowed by stability), the distribution is supported on one of the outcomes and prediction is deterministic. Between these extremes, the difference in values of $\phi$ measures the influence of $t$-th coordinate on the potential function $\phi$, where the past outcomes $y_1,\ldots,y_{t-1}$ have been fixed and future is uniformly-random. We emphasize that Rademacher random variables for future rounds are purely an outcome of the minimax analysis, as we assume no generative mechanism for the sequence $\by$.
	
For $k>2$, the randomized algorithm is defined on round $t$ by a distribution\footnote{We hope the difference in the meaning of $q_t^*$ as a distribution (for $k>2$) vs a mean of a distribution (for $k=2$) will not cause confusion.} on $\{1,\ldots,k\}$, and the optimal choice is given by
\begin{eqnarray}
	\label{eq:algo_multi}
	q_t^*(y_{1:t-1}) = \argmin{q\in \Delta_k}\max_{j\in[k]} \Big\{ -q^\tr {\boldsymbol e}_j - n\En\phi(y_{1:t-1},j,u_{t+1:n}) \Big\} 
\end{eqnarray}
where the expectation is with respect to $u_{t+1},\ldots,u_n$, each independent uniform on $[k]$. Given that the values $\En\phi(y_{1:t-1},j,u_{t+1:n})$ have been computed for each $j$, the minimization in \eqref{eq:algo_multi} is performed by a simple water-filling $O(k)$-time algorithm which can be found in the proof of Lemma~\ref{lem:cover_multi} (see also  \cite{rakhlin2016bistro}). The actual prediction is then a random draw $\pred_t\sim q_t^*$.

\subsection{Randomization}

Computing the expectations in \eqref{eq:algo_2} and \eqref{eq:algo_multi} may be costly. Thankfully, a doubly-randomized strategy works by drawing a random sequence per iteration. For binary prediction, the algorithm on round $t$ becomes: draw independent Rademacher random variables $\varepsilon_{t+1},\ldots,\varepsilon_n$, compute
	\begin{eqnarray}
		\label{eq:algo_2_randomized}
		\widetilde{q}_t^*(y_{1:t-1},\varepsilon_{t+1:n}) =n[ \phi(y_{1:t-1},-1,\varepsilon_{t+1:n})- \phi(y_{1:t-1},+1,\varepsilon_{t+1:n})] 
	\end{eqnarray}
	and draw $\pred_t$ from the distribution on $\{\pm1\}$ with the mean $\widetilde{q}_t^*$. This randomized strategy was essentially proposed in \cite{cesa2011efficient}. 
	
	For $k>2$, we draw uniform independent $u_{t+1},\ldots,u_n$, solve for
	\begin{eqnarray}
		\label{eq:algo_multi_randomized}
		\widetilde{q}_t^*(y_{1:t-1},u_{t+1:n}) = \argmin{q\in \Delta_k}\max_{j\in[k]} \Big\{ -q^\tr {\boldsymbol e}_j - n\phi(y_{1:t-1},j,u_{t+1:n}) \Big\} 
	\end{eqnarray}
	and then draw prediction $\pred_t$ from the distribution $\widetilde{q}_t^*$.
	
	\begin{lemma}
		\label{lem:random_play_bin_multi}
		The doubly-randomized strategies \eqref{eq:algo_2_randomized} and \eqref{eq:algo_multi_randomized} enjoy, in expectation, the same mistake bounds as, respectively, \eqref{eq:algo_2} and \eqref{eq:algo_multi}.
	\end{lemma}
	In the binary prediction case, the proof of Lemma~\ref{lem:random_play_bin_multi} is immediate by the linearity of expectation. The analogous argument for \eqref{eq:algo_multi_randomized} is more tricky and follows from a more general \emph{random playout} technique introduced in \cite{rakhlin2012relax}. This technique also yields a proof of Lemma~\ref{lem:random_play_bin_multi} (for both binary and multi-class cases) for \emph{adaptively chosen} sequences of outcomes, an issue we have not yet discussed (see also Section~\ref{sec:cov} below).

\section{Relaxations and the power of improper learning}
\label{sec:rel}

In the rest of the tutorial, we focus on the binary prediction problem for simplicity. Recall that the computation in \eqref{eq:algo_2} involves drawing random bits $\varepsilon_{t+1},\ldots,\varepsilon_n$ and evaluating the $\phi$ function on two neighboring vertices of the hypercube. If $\phi$ is defined as 
\begin{align}
	\label{eq:reldham0}
	\phi(\by) = \dham(\by,F) + C_{n,F},
\end{align}
then computing \eqref{eq:algo_2} or \eqref{eq:algo_2_randomized} involves comparing two distances to the set $F$, as shown in Figure~\ref{fig:graphics_hypercube-distance-from-two}. Note that the knowledge of $C_{n,F}$ is not needed, as this value cancels off in the difference. 

	\begin{figure}[h]
	  \centering
	    \includegraphics[width=.5\textwidth]{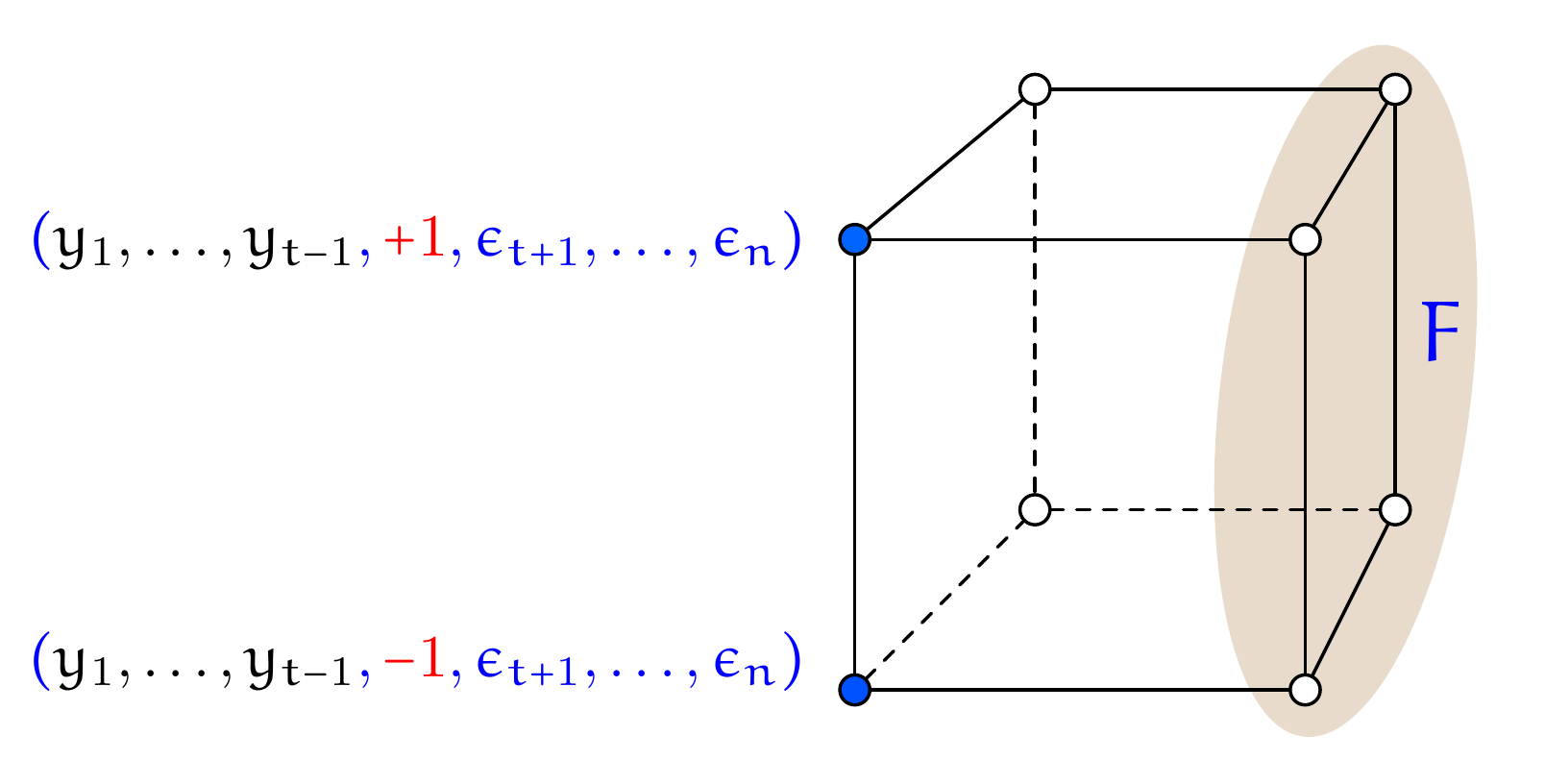}
	  \caption{Randomized strategy involves computing the difference of normalized Hamming distances from neighboring vertices to $F$.}
	  \label{fig:graphics_hypercube-distance-from-two}
	\end{figure}

Since
	\begin{align}
		\label{eq:dham_ext}
		\dham(\by,F) = \min_{\bw\in F} \frac{1}{n}\sum_{t=1}^n \ind{w_t\neq y_t} = \frac{1}{2}-\frac{1}{2n}\max_{\bw\in F} \bw^\tr \by,
	\end{align}
	we may extend the function $\dham(\by,F)$ to any $F\subseteq [-1,1]^n$ by defining it as the right-hand side of \eqref{eq:dham_ext}. The extended $\dham(\by,F)$ is still stable in the sense of \eqref{eq:stab}.
	
	Suppose that calculating the distance $\dham(\by,F)$ is computationally expensive, due to the combinatorial nature of $F\subset \{\pm1\}^n$. Let $F'\subseteq [-1,1]^n$ be a set containing $F$, and suppose that $\dham(\by,F')$ is easier to compute. The following observation is immediate:
	\begin{observation} 
		\label{obs:approx}
		Let $F\subseteq F'\subseteq [-1,1]^n$. Algorithms \eqref{eq:algo_2} and \eqref{eq:algo_2_randomized} with 
		$$\phi'(\by) = \dham(\by,F') + \Rad(F'),$$
	enjoy a mistake bound
	\begin{align}
		\label{eq:obs_mist_F}
		\En\left[ \frac{1}{n}\sum_{t=1}^n \ind{\pred_t\neq y_t} \right] \leq \min_{\bw\in F} \left[ \frac{1}{n}\sum_{t=1}^n \ind{w_t \neq y_t} \right] + \alpha \times  \Rad(F)
	\end{align}
	for $\alpha \geq \Rad(F')/\Rad(F)$.
	\end{observation}
	\begin{proof}
		By construction, $\phi'$ is achievable, and 
		\begin{align}
			\label{eq:reldham1}
			\muca(\by)\leq \phi'(\by) \leq \dham(\by,F) + \Rad(F') \leq  \dham(\by,F) + \alpha\times\Rad(F) .
		\end{align}
	\end{proof}
	By relaxing the set $F$ to a larger set $F'$, we may gain in computation while suffering a multiplicative factor $\alpha$ in the Rademacher complexity. Crucially, this factor \emph{does not multiply the Hamming distance} but only the term $\Rad(F)$. The latter is typically of lower order and diminishing with $n$. We may summarize the observation as 
	\begin{center}
		 $\Bigg($\stackanchor{online }{mistakes}$\Bigg)$ ~$\leq$~ $\Bigg($\stackanchor{offline combinatorial}{benchmark}$\Bigg)$ ~$+$~  $\alpha \times \Bigg($\stackanchor{additive $o(1)$}{error}$\Bigg)$~~.
	\end{center}
	There is a reason we belabor this simple observation. In the literature on online learning, it has been noted that one may guarantee 
	\begin{center}
		 $\Bigg($\stackanchor{online }{mistakes}$\Bigg)$ ~$\leq$~ $\alpha \times \Bigg($\stackanchor{offline combinatorial}{benchmark}$\Bigg)$ ~$+$~  $\Bigg($\stackanchor{additive $o(1)$}{error}$\Bigg)$~~.
	\end{center}
	when one has a multiplicative approximation algorithm for the benchmark. The performance of the algorithm in this case is compared to
	$$\alpha \times \min_{\bw\in F} \left[ \frac{1}{n}\sum_{t=1}^n \ind{w_t \neq y_t} \right].$$
	However, the bound easily becomes vacuous (say, the error rate of the offline benchmark is $5\%$ and the multiplicative factor is a constant or logarithmic in $n$). The version where $\alpha$ only enters the remainder term seems much more attractive.
	
	The key to obtaining \eqref{eq:obs_mist_F} is the \textbf{improper} nature of the prediction methods \eqref{eq:algo_2} or \eqref{eq:algo_2_randomized}: the prediction $\pred_t$ need not be in any way consistent with any of the models in $F$. Informally:
	\begin{quote}
		\sf
		Using improper learning methods, it may be possible to predict as well as a combinatorial benchmark (plus a lower-order term) even when computing this very benchmark given all the data is NP-hard.
	\end{quote}
	
	To make the statement more meaningful, we will show that $\alpha$ can be upper bounded---for some interesting examples---in a way that does not render the mistake guarantee vacuous. We start with a simple example in Section~\ref{sec:rad_spectrum}, and then present a more complex machinery based on Constraint Satisfaction in Section~\ref{sec:csp}. 

\subsection{Example: node classification}
\label{sec:rad_spectrum}

Consider the example in Section~\ref{sec:community}, and suppose, additionally, that the undirected graph $G=(V,E,W)$ has weights on edges. The weight on edge $(i,j)$ is denoted by $W_{i,j}$, and $W_{i,j} \equiv 0$ when $(i,j)$ is not an edge. Positive and negative edges may model friend/foe or trust/distrust networks.

We define $F_\kappa$ as in \eqref{eq:def_F_laplacian}, with the understanding that $L$ is now the weighted graph Laplacian: $L=D-W$ with $D$ the diagonal matrix,  $D_{i,i} = \sum_{j=1}^n |W_{i,j}|$. As before, set $\phi(\by) = \dham(\by,F_\kappa) + C_{n,F_\kappa}$. Why would evaluation of $\dham(\cdot, F_\kappa)$ be computationally hard? First, if we can evaluate $\phi$ for any $\kappa$, we can also find the value 
\begin{align}
	\label{eq:lap_iqp}
	\min_{\bw\in\{\pm1\}^n} \bw^\tr L \bw.
\end{align}	
However, if all the edge weights are $-1$, then \eqref{eq:lap_iqp} becomes
\begin{align}
	\label{eq:lap_iqp_minuncut}
	 \min_{\bw\in\{\pm1\}^n} \sum_{(i,j)\in E} 1+ w_i w_j,
\end{align} 
which may be recognized as the value of {\sf MinUnCut}, an NP-hard problem in general. Hence, we cannot hope to evaluate the Hamming distance to $F_\kappa$ exactly. Our first impulse is to approximate the value in \eqref{eq:lap_iqp}; in the case of {\sf MinUnCut} this can be done with a multiplicative factor of $O(\sqrt{\log n})$. However, it is not clear how to turn such a multiplicative approximation into a mistake bound with the factor $1$ in front of the combinatorial benchmark. Yet, such a bound \emph{is} possible, as we show next.


Following Observation~\ref{obs:approx}, we set
\begin{align}
	F'_\kappa = \left\{ \bw\in [-1,1]^n ~:~ \bw^\tr L \bw \leq \kappa \right\}
\end{align}
and extend
\begin{align}
	\dham(\by,F'_\kappa) = \frac{1}{2}-\frac{1}{2n}\max_{\bw\in F'_\kappa} \bw^\tr \by,
\end{align}
as in \eqref{eq:dham_ext}. Note that evaluating $\dham(\by,F'_\kappa)$ amounts to maximization of a linear function subject to a quadratic constraint $\bw^\tr L \bw \leq \kappa$ and a box constraint $\bw\in [-1,1]^n$. This can be accomplished with standard optimization toolboxes. 

It remains to estimate $\Rad(F'_\kappa)$. To this end, notice that
\begin{align}
	F'_\kappa \subset \left\{\bw\in\reals^n: \bw^\tr \bw \leq n, \bw^\tr L \bw\leq \kappa \right\} \subset \left\{\bw\in\reals^n: \bw^\tr M \bw \leq 1 \right\}
\end{align}
with $M=\frac{1}{2n}I + \frac{1}{2\kappa}L$. Hence, we can upper bound
\begin{align}
	\label{eq:rad_spectrum}
	2n \cdot \Rad(F'_\kappa) = \En_{\bepsilon} \max_{\bw\in F'_\kappa} \bw^\tr \bepsilon \leq \En_{\bepsilon}  \max_{\bw^\tr M \bw\leq 1} \bw^\tr \bepsilon \leq \En_{\bepsilon}  \sqrt{\bepsilon^\tr M^{-1}\bepsilon} \leq \sqrt{\sum_{j=1}^n \lambda_j(M)^{-1}},
\end{align}
where $\lambda_j(M)$ is the $j$th eigenvalue of $M$. The upper bound in \eqref{eq:rad_spectrum} depends on the spectrum of the underlying network's Laplacian with an added regularization term $\frac{1}{2n}I$. It is an interesting research direction to find tighter upper bounds, especially when the social network evolves according to a random process.

To summarize, we relaxed $F_\kappa$ to a larger set $F'_\kappa$, for which computation can be performed with an off-the-shelf optimization toolbox. Further, we derived a (rather crude) upper bound on the Rademacher averages of this set. In our calculations, however, we did not obtain an upper bound on the multiplicative gap $\alpha$ between the original Rademacher averages $\Rad(F_\kappa)$ and the larger value $\Rad(F'_\kappa)$. Hence, the price we paid for efficiently computable solutions remains unknown. In the next section, we present a generic way to glean this payment from known approximations to integer programs.

\section{Computational hierarchies}
\label{sec:csp}

Unlike the rest of the tutorial, this section is not self-contained. Our aim is to sketch the technique in \cite{RakSri15hierarchies}, hiding the details under the rug. We also refer the reader to the literature on approximation algorithms based on semidefinite and linear programming (see e.g. \cite{gartner2012approximation} and references therein).

\subsection{Relaxing the optimization problem}

Consider the set
\begin{align}
	\label{eq:def_F_constr}
	F_\kappa = \left\{ \bw\in \{\pm1\}^n ~:~ \Constr(\bw) \leq \kappa \right\}
\end{align}
for some $\Constr:\{\pm1\}^n\to \reals$, which we call a \emph{constraint}. The definitions \eqref{eq:def_F_laplacian} and \eqref{eq:lap2w} in terms of graph Laplacian and weighted graph Laplacian are two examples of such a definition. A more general example is Constraint Satisfaction: let ${\mathscr C}$ be a collection of functions ${\mathfrak z}:\{\pm1\}^n\to \reals$ and set $$\Constr(\bw) = \sum_{{\mathfrak z}\in{\mathscr C}} {\mathfrak z}(\bw).$$

Recall that computing a prediction amounts to evaluating the weighted Hamming distance from $\by\in\{\pm1\}^n$ to $F_\kappa$, which---in view of \eqref{eq:dham_ext}---is equivalent to finding the value  
	\begin{align}
		\label{eq:opt1}
		\OPT_1(\kappa,\by) ~\deq~ \max_{\bw\in F_\kappa,~ \bw^\tr \by\geq \beta}~ \beta
	\end{align}
	and then setting 
	\begin{align}
		\label{eq:phidef_OPT1}
		\phi(\by) = \frac{1}{2} - \frac{1}{2n} \OPT_1(\kappa,\by) + C_{n,F_\kappa}.
	\end{align}
	Observation~\ref{obs:approx} suggests that if $\OPT_1$ cannot be easily computed for the set $F_\kappa$, we should aim to find a larger set $F'_\kappa$ for which this optimization is easier. A twist here is that we will not write down the definition of $F'_\kappa$ explicitly (although it can be understood as a projection of a certain higher-dimensional object). Instead, let us replace $\OPT_1(\kappa,\by)$ in \eqref{eq:phidef_OPT1} by a value $\SDP_1(\kappa,\by)$ of some other optimization problem (to be specified in a bit) and set
	\begin{align}
		\label{eq:phiprime_sdp}
		\phi'(\by) = \frac{1}{2} - \frac{1}{2n} \SDP_1(\kappa,\by) + C_{n,F'_\kappa}.
	\end{align}
	As before, the condition $\En\phi'(\bepsilon)\geq 1/2$ for achievability of this function implies that the smallest value of the constant is
	\begin{align}
		\label{eq:cprime}
		C_{n,F'_\kappa} = \frac{1}{2n} \En\left[ \SDP_1(\kappa,\bepsilon) \right].
	\end{align}
	Before defining $\SDP_1$, let us state a version of Observation~\ref{obs:approx}:
	\begin{observation}
		\label{obs:approx2}
		If for any $\by\in\{\pm1\}^n$ it holds that 
		\begin{align}
			\label{eq:wish_integr_gap}
			\SDP_1(\kappa,\by) \leq \OPT_1(\alpha \cdot \kappa,\by)
		\end{align} 
		for some $\alpha\in\reals$, then using the algorithm \eqref{eq:algo_2} or \eqref{eq:algo_2_randomized} with $\phi'$ in \eqref{eq:phiprime_sdp}, we guarantee
a regret bound of
	\begin{align}
		\label{eq:obs_mist_F_kappa}
		\En\left[ \frac{1}{n}\sum_{t=1}^n \ind{\pred_t\neq y_t} \right] \leq \min_{\bw\in F_\kappa} \left[ \frac{1}{n}\sum_{t=1}^n \ind{w_t \neq y_t} \right] + \Rad(F_{\alpha\cdot \kappa})
	\end{align}
	for any sequence $\by$.
	\end{observation}
	\begin{proof}
		Immediate from \eqref{eq:cprime} and \eqref{eq:wish_integr_gap}, and the fact that $F'_\kappa$ defined by $\SDP_1$ contains $F_\kappa$. 
	\end{proof}

	We now define $\SDP_1$, along with two more auxiliary optimization problems, and prove \eqref{eq:wish_integr_gap}. 
	
	\subsection{Setting up auxiliary optimization problems}

	$\OPT_1$ is an optimization of a linear objective over vertices of the hypercube, under the restriction $\Constr(\x)\leq \kappa$. In the literature, much effort has been devoted to analyzing a dual problem: minimization of $\Constr(\x)$, possibly subject to a linear constraint. Our plan of attack is to define dual auxiliary optimization formulations, then use ``integrality gap'' for these problems, and pass back to the primal objective $\OPT_1$ in order to prove \eqref{eq:wish_integr_gap}.
	
	Let us define the set of probability distributions on those vertices of the hypercube that yield the value of at least $\beta$ for the linear objective:
	\begin{align}
		\label{eq:delta2}
		\cD(\beta,\by) ~\deq~ \Delta(\{\bw\in \{\pm1\}^n,~ \bw^\tr \by \geq \beta\}).
	\end{align}	
Define the optimization problem
	\begin{align}
		\label{eq:opt2}
		\OPT_2(\beta,\by) ~\deq~ \min_{\bw\in \{\pm1\}^n,~ \bw^\tr \by \geq \beta}~ \Constr(\bw) ~=~ \min_{p\in \cD(\beta,\by)} \En_{\bz\sim p} \Constr(\bz),
	\end{align}
the minimum constraint value achievable on the vertices of the hypercube, given that the linear objective value is at least $\beta$. The second equality in \eqref{eq:opt2} holds true because the minimum of a linear (in $p$) objective is attained at a singleton, a vertex of the hypercube.

Both \eqref{eq:opt1} and \eqref{eq:opt2} are combinatorial optimization problems, which may be computationally intractable. A common approach to approximating these hard problems is to pass from distributions to \emph{pseudo-distributions}. Roughly speaking, a pseudo-distribution at ``level $r$'' only behaves like a distribution for tuples of variables of size up to $r$. Associated to a pseudo-distribution is a notion of a pseudo-expectation, denoted by $\widehat{\En}$. We refer to \cite{boaz14notes,chlamtac2012convex,rothvoss2013lasserre} for details.

Let $\widehat{\cD}(\beta,\by)$ be the set of suitably defined pseudo-distributions with the property 
$$\widehat{\cD}(\beta,\by)\subseteq \widehat{\cD}(\beta',\by)$$
whenever $\beta'\leq \beta$ (see \cite{RakSri15hierarchies} for the precise definitions of these sets in terms of semidefinite programs). Define a relaxation of \eqref{eq:opt2} as
	\begin{align}
		\SDP_2(\beta,\by) ~\deq~ \min_{\widehat{p}\in\widehat{\cD}(\beta,\by)}~ \widehat{\En}_{\bz\sim \widehat{p}} \Constr(\bz)
	\end{align}
	and let 
	\begin{align}
		\SDP_1(\kappa,\by) ~\deq~ \max_{\exists \widehat{p} \in\widehat{\cD}(\beta,\by) ~\text{s.t.}~  \widehat{\En}_{\bz\sim \widehat{p}} \Constr(\bz)\leq \kappa}~ \beta ~.
	\end{align}
We write ``SDP'' here because relaxations we have in mind arise from semidefinite relaxations, but the arguments below are generic and hold for  other approximations of combinatorial optimization problems.

The \emph{integrality gap} of the dual formulation is 
\begin{align}
	\label{eq:intgap}
	\alpha(\by) \deq \max_{\beta}~ \frac{\OPT_2(\beta,\by)}{\SDP_2(\beta,\by)},
\end{align}
with $\alpha \deq \max_{\by} \alpha(\by)$.
We emphasize that we define the gap for the dual problems. Next, we show that the gap appears when relating $\SDP_1$ to $\OPT_1$.

\begin{lemma}
	\label{lem:int_gap}
	For any $\by\in\{\pm1\}^n$,
	\begin{align}
		\SDP_1(\kappa, \by)\leq \OPT_1(\alpha(\by) \cdot \kappa, \by).
	\end{align}
\end{lemma}
\begin{proof}[\textbf{Proof of Lemma~\ref{lem:int_gap}}]
	Let us fix $\by$ and for brevity omit it from $\OPT$ and $\SDP$ definitions. First, it holds that
\begin{align}
	\label{eq:sdp2sdp1}
	\SDP_2(\SDP_1(\kappa))\leq \kappa.
\end{align}
Indeed, $\SDP_1(\kappa)$ is the value of $\beta$ that guarantees existence of some pseudo-distribution $\widehat{p}$ in $\widehat{\cD}(\beta,\by)$ that respects the constraint of $\kappa$ on average. Hence, for this value of $\beta$, the minimum in $\SDP_2(\beta)$ will include this pseudo-distribution, and, hence, the value of the minimum is at most $\kappa$. By the same token,
\begin{align}
	\label{eq:opt1opt2}
	\OPT_1(\OPT_2(\beta))\geq \beta.
\end{align}
Indeed, $\OPT_2(\beta)$ is the value of $\Constr(\bw)$ achieved by some $\bw\in\{\pm1\}^n$ satisfying $\bw^\tr\by\geq \beta$. The maximization in $\OPT_1$ includes this $\bw$ by the definition of $F_\kappa$, and, hence, the maximum is larger than $\beta$.

Next, by the definition of the integrality gap and \eqref{eq:sdp2sdp1},
\begin{align}
	\label{eq:opt2sdp1}
	\OPT_2(\SDP_1(\kappa))\leq \alpha \cdot \SDP_2(\SDP_1(\kappa))\leq \alpha \cdot \kappa.
\end{align}
By \eqref{eq:opt1opt2} and \eqref{eq:opt2sdp1},
\begin{align}
	\SDP_1(\kappa)\leq \OPT_1(\OPT_2(\SDP_1(\kappa))) \leq \OPT_1(\alpha \cdot \kappa)
\end{align}
because the value of $\OPT_1(\kappa)$ is nondecreasing in $\kappa$. 
\end{proof}

It remains to provide concrete examples where $\alpha(\by)$ is bounded in a non-trivial manner. 

\subsection{Back to node classification}

Recall the node classification example discussed in Section~\ref{sec:rad_spectrum}. In view of Lemma~\ref{lem:int_gap}, to conclude the mistake bound \eqref{eq:obs_mist_F_kappa} we only need to get an estimate on $\alpha$. 
For the case $\Constr(\bw) = \bw^\tr L \bw$, the integrality gap defined in \eqref{eq:intgap} is the ratio of an integer quadratic program (IQP) subject to linear constraint \eqref{eq:opt2} and its relaxed version. We turn to \cite[Theorem 6.1]{guruswami2013rounding}, which tells us that within $O(r)$ levels of Lasserre hierarchy one can solve the IQP subject to linear constraints with the gap of at most
$$O(\max\{\lambda_r^{-1},1\})$$ 
where $\lambda_r$ is the $r$-th eigenvalue of the normalized graph Laplacian. One can verify that $\Rad(F_\kappa)$ grows as $\sqrt{\kappa}$, and thus we essentially pay an extra factor of  $O(\max\{\lambda_r^{-1/2},1\})$ for having a polynomial-time algorithm, with the computational complexity of $O(n^{r})$.

There are several points worth emphasizing:
\begin{itemize}
	\item As another manifestation of improper learning, the prediction algorithm does not need to ``round the solution.'' The integrality gap only appears in the mistake analysis. This means that any improvement in the gap analysis of semidefinite or other relaxations immediately translates to a tighter mistake bound for the same prediction algorithm.
	\item It is the expected gap (with respect to a random direction $\bepsilon$) that enters the bound, a quantity that may be smaller than the worst-case gap. It would be interesting to quantify this gain.
\end{itemize}
As an alternative to definition \eqref{eq:def_F_constr}, one may combine the linear objective and the constraint in a single ``penalized'' form. Such an approach allows one to use Metric Labeling approximation algorithms \cite{kleinberg2002approximation}, and we refer to \cite{RakSri15hierarchies} for more details.

\section{Incorporating covariates}
\label{sec:cov}

In most applications of online prediction, some additional side information is available to the decision-maker before she makes the prediction. Consider the following generalization of the problem introduced in Section~\ref{sec:basics}. At time $t=1,\ldots,n$, the forecaster observes side information $x_t\in\X$, makes a forecast $\pred_t\in\{\pm1\}$ (based on the history $y_1,\ldots,y_{t-1}$ and $x_1,\ldots,x_t$), and then the value $y_t$ is observed.

Let $\phi$ be a function of two sequences: $\phi:\X\times \{\pm1\}^n \to \reals$. The function $\phi$ is stable if
\begin{align}
	\label{eq:lec12:smoothness}
	|\phi(\bx; y_{1:t-1},+1,y_{t+1:n})-\phi(\bx; y_{1:t-1},-1,y_{t+1:n})|\leq 1/n,
\end{align}
where $\bx=(x_1,\ldots,x_n)$.
We will prove the following generalization of Lemma~\ref{lem:cover}.

\begin{lemma}
	\label{lem:phi_x_y}
	Let $\phi:\X\times \{\pm1\}^n\to\reals$ be stable, and suppose that $x_t$'s are i.i.d. Then there exists a prediction strategy  such that
	\begin{align}
		\label{eq:lec12:mistake_bound_iid}
		\forall \by\in\{\pm1\}^n,~~~~ \En\left[ \frac{1}{n}\sum_{t=1}^n\ind{\pred_t\neq y_t}\right] \leq \En\phi(\bx;\by)
	\end{align}
	if and only if
	\begin{align}
		\label{eq:lec12:half_cond_phi}
		\En\phi(\bx;\bepsilon)\geq 1/2.
	\end{align}
\end{lemma}
Above, the expectation on the left-hand side of \eqref{eq:lec12:mistake_bound_iid} is over the randomization of the algorithm and the $x$'s, while on the right-hand side the expectation is over the $x$'s. In \eqref{eq:lec12:half_cond_phi}, the expectation is both over the $x$'s and over the independent Rademacher random variables.

An attentive reader will notice that the guarantee \eqref{eq:lec12:mistake_bound_iid} is not very interesting because $x_t$'s are chosen independently while $y_t$'s are fixed ahead of time. The issue would be resolved if $y_t$'s could be chosen by Nature \emph{after} seeing $x_t$. Let us call such a Nature \emph{semi-adaptive}, and reserve the word \emph{adaptive} for Nature that chooses $y_t$'s based also on the full history of $\{(x_s,y_s,\pred_s)\}_{s=1}^{t-1}$ (including learner's predictions). 

\begin{lemma}
	\label{lem:adaptive_phi_x_y}
	Lemma~\ref{lem:phi_x_y} holds for an adaptive Nature. That is, \eqref{eq:lec12:half_cond_phi} is equivalent to existence of a prediction algorithm (given in \eqref{eq:opt_rand_strategy_hybrid} below) with
	\begin{align}
		\label{eq:lec12:mistake_bound_iid_adaptive}
		\En\left[ \frac{1}{n}\sum_{t=1}^n\ind{\pred_t\neq y_t}\right] \leq \En\phi(\bx;\by),
	\end{align}
	where each $y_t\in\{\pm1\}$ may be chosen arbitrarily by Nature  based on the history $\{(x_s,y_s,\pred_s)\}_{s=1}^{t-1}$ and $x_t$.
\end{lemma}
Inspecting the proof of Lemma~\ref{lem:adaptive_phi_x_y}, we see that the optimal doubly-randomized strategy is to draw $x_{t+1:n}$ and $\varepsilon_{t+1:n}$ and then set the mean of the distribution to be 
\begin{eqnarray}
	\label{eq:opt_rand_strategy_hybrid}
	&\widetilde{q}_t^*(\bx,y_{1:t-1},\varepsilon_{t+1:n}) = n\left[\phi(\bx;y_{1:t-1},-1,\varepsilon_{t+1:n})-\phi(\bx;y_{1:t-1},+1,\varepsilon_{t+1:n})\right].
\end{eqnarray}
Notice that the coordinates $x_{1:t}$ of $\bx$ are the actual observations, while $x_{t+1:n}$ are hallucinated. If $x$'s are i.i.d., these hypothetical observations are available, for instance, if one has access to a pool of unlabeled data. In fact, the statement of Lemma~\ref{lem:adaptive_phi_x_y} holds verbatim for any stochastic process governing evolution of $x$'s, as long as we can ``roll out the future'' according to this process.

Finally, we state one more extension of Cover's result, lifting any stochastic assumptions on the generation of the $x$'s.
\begin{lemma}
	\label{lem:adaptive_phi_x_y_fully_advers}
	Let $\phi:\X^n\times\{\pm1\}^n\to\reals$ be stable, and assume that both $x_t$ and $y_t$ are chosen by Nature adversarially and adaptively. Then existence of a strategy with 
	$$\En\left[ \frac{1}{n}\sum_{t=1}^n\ind{\pred_t\neq y_t}\right] \leq \En\phi(\bx;\by)$$
	is equivalent to
	\begin{align}
		\label{eq:con_phi_tree}
		\forall \bz,~~ \En_{\bepsilon} \phi(\bz_1,\bz_2(\varepsilon_1),\ldots,\bz_n(\varepsilon_{1:n-1}); \bepsilon) \geq \frac{1}{2},
	\end{align}
	where $\bz=(\bz_1,\ldots,\bz_n)$ is an $\X$-valued predictable process with respect to the dyadic filtration $\{\sigma(\varepsilon_1,\ldots,\varepsilon_t)\}_{t\geq 0}$ generated by i.i.d. Rademacher $\varepsilon_1,\ldots,\varepsilon_n$.
\end{lemma}

In Section~\ref{sec:basics}, we introduced a ``canonical'' way to define $\phi(\by)$ for the case of no side information. The analogous canonical definition that takes side information into account is as follows. Let $\F$ be a class of functions $\X\to\{\pm1\}$. If $\F$ is chosen well, one of the functions in this class will explain, approximately, the relationship between $x_t$'s and $y_t$'s. It is then natural to take the projection 
$$F|_{\bx} = \{(f(x_1),\ldots,f(x_n)): f\in\F\},$$
the set of vertices of the hypercube achieved by evaluating some $f\in\F$ on the given data. We may now define
$$\phi(\bx;\by) = \dham(\by, F|_{\bx}) + C_{n,F|_{\bx}},$$
as before. The function $\phi$ defined in this way is indeed small if $\by$ is close to the values of some $f\in\F$ on the data. 

It remains to give an expression for $C_{n,F|_{\bx}}$. For the i.i.d. side information case of Lemma~\ref{lem:adaptive_phi_x_y}, the condition $\En\phi(\bx;\bepsilon)\geq 1/2$ means that the smallest value of $C_{n,F|_{\bx}}$ ensuring achievability is
$$\frac{1}{2n}\En_{\bepsilon,\bx}\left[\sup_{f\in\F} \sum_{t=1}^n \varepsilon_t f(x_t)\right],$$
the Rademacher averages of $\F$. For the adversarial case of Lemma~\ref{lem:adaptive_phi_x_y_fully_advers}, condition \eqref{eq:con_phi_tree} means that the smallest value of $C_{n,F|_{\bx}}$ is
$$\sup_{\bz} \Rad^{\text{seq}}(\F, \bz),$$
where 
$$\Rad^{\text{seq}}(\F, \bz) \deq \frac{1}{2n}\En_{\bepsilon}\left[\sup_{f\in\F} \sum_{t=1}^n \varepsilon_t f(\bz_t(\varepsilon_{1:t-1}))\right]$$
is the \emph{sequential} Rademacher complexity \cite{rakhlin2015sequential}.

\section{Discussion and Research Directions}

The prediction results discussed in this tutorial hold for arbitrary sequences -- even for those chosen adversarially and adaptively in response to forecaster's past predictions. Treating the prediction problem as a multi-stage game against Nature has been very fruitful, both for the theoretical analysis and for the algorithmic development. Even though we discuss maliciously chosen sequences, it is certainly not our aim to paint any prediction problem as adversarial. Rather, we view the ``individual sequence'' results as being \emph{robust} and applicable in situations when modeling the underlying stochastic process is difficult. For instance, one may try to model the node prediction problem described in Section~\ref{sec:community} probabilistically---e.g. as a Stochastic Block Model---but such a model is unlikely to be true in the real world. Of course, the ultimate test is how the two approaches perform on real data. In the node classification example, the methods discussed in this tutorial performed very well in our own experiments, often surpassing the performance of more classical machine learning methods. Perhaps it is worth emphasizing that the prediction algorithms developed here are very distinct from these classical methods, and, if anything, this tutorial serves the purpose of enlarging the algorithmic toolkit.

We presented some very basic ideas, only scratching the surface of what is possible. Among some of the most interesting (to us) and promising research directions are:
\begin{itemize}
	\item Develop linear or sublinear time methods for solving prediction problems on large-scale graphs. 
	\item Run more experiments on real-world data and explore the types of functions $\phi$ that lead to good prediction performance. 	
	\item Develop partial-information versions of the problem. Some initial steps for contextual bandits were taken in \cite{rakhlin2016bistro,SyrgkanisLKS16}.
	\item Analyze the setting of constrained sequences. That is, develop methods when Nature is not fully adversarial, yet also not i.i.d. 
	\item Develop efficient prediction methods that go beyond i.i.d. covariates.
\end{itemize}
For more additional questions or clarifications, please feel free to email us.

\appendix
\section{Proofs}
\begin{proof}[\textbf{Proof of Lemma~\ref{lem:cover_multi}}]
	Define functions $\Rel_t:[k]^t\to\reals$ as
	$$\Rel_n(y_1,\ldots,y_n) = -\phi(y_1,\ldots,y_n)$$
	and
	\begin{align}
		\label{eq:rel_t}
		\Rel_{t-1}(y_1,\ldots,y_{t-1}) = \En_{y_t\sim \text{Unif}[k]}\Rel_{t}(y_1,\ldots,y_{t}) +  \frac{1}{n}\left(1-\frac{1}{k}\right),
	\end{align}
	with $\Rel_{0}(\emptyset)$ being a constant.
	We desire to prove that there is an algorithm such that 
	$$\forall \by\in[k]^n,~~~~ \En\left[ \frac{1}{n}\sum_{t=1}^n \ind{\pred_t\neq y_t} \right] -  \phi(y_1,\ldots,y_n) = 0.$$
	Consider the last time step $n$ and write the above expression as
	\begin{align}
		\label{eq:minm0}
		\En\left[ \frac{1}{n}\sum_{t=1}^{n-1} \ind{\pred_t\neq y_t} + \frac{1}{n}\ind{\pred_n\neq y_n} +  \Rel_n(y_1,\ldots,y_n) \right].
	\end{align}
	Let $\En_{n-1}$ denote the conditional expectation given $\pred_1,\ldots,\pred_{n-1}$. We shall prove that there exists a randomized strategy for the last step such that for any $y_n\in [k]$,
	\begin{align}
		\En_{n-1}\left[\frac{1}{n}\ind{\pred_n\neq y_n} \right] +  \Rel_n(y_1,\ldots,y_n)  = \Rel_{n-1} (y_1,\ldots,y_{n-1}).
	\end{align}
	This last statement is translated as 
	\begin{align}
		\label{eq:minm1}
		\min_{q_n\in\Delta_k}\max_{y_n\in[k]} \left\{ \En_{n-1}\left[\frac{1}{n}\ind{\pred_n\neq y_n}\right] +  \Rel_n(y_1,\ldots,y_n) \right\} = \Rel_{n-1} (y_1,\ldots,y_{n-1}).
	\end{align}
	Writing $\ind{\pred_n\neq y_n}=1-\boldsymbol{e}_{\pred_n}^\tr \boldsymbol{e}_{y_n}$, the left-hand side of \eqref{eq:minm1} is 
	\begin{align}
		\label{eq:minm2}
		\frac{1}{n} \min_{q_n\in\Delta_k}\max_{y_n\in[k]} \left\{ 1-q_n^\tr \boldsymbol{e}_{y_n} +  n\Rel_n(y_1,\ldots,y_n) \right\}.
	\end{align}
	The stability condition \eqref{eq:stab_multiclass} means that we can choose $q_n$ to \emph{equalize} the choices of $y_n$. Let $\psi(1),\ldots,\psi(k)$ be the sorted values of $$n\Rel_n(y_1,\ldots,y_{n-1},1),\ldots,n\Rel_n(y_1,\ldots,y_{n-1},k),$$ in non-increasing order. In view of the stability condition,
	$$\sum_{i=1}^k (\psi(i)-\psi(k)) \leq 1.$$
	Hence, $q_n$ can be chosen so that all
	$\psi(i)-q_n(i)$
	have the same value (see Figure~\ref{fig:graphics_water-filling-proof}). One can check that this is the minimizing choice for $q_n$.
	\begin{figure}[t]
	  \centering
	    \includegraphics[width=.4\textwidth]{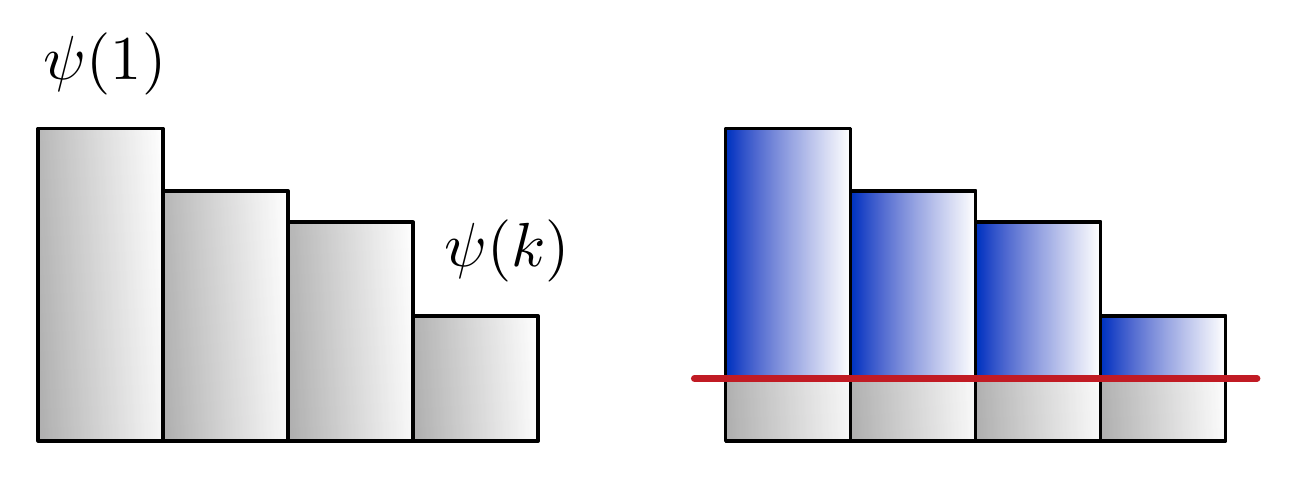}
	  \caption{Under the stability condition, water-filling is optimal.}
	  \label{fig:graphics_water-filling-proof}
	\end{figure}
	Let $q_n^*$ denote this optimal choice. The common value of $\psi(i)-q_n^*(i)$ can then be written as
	$$\psi(k) - \frac{1}{k}\left(1-\sum_{i=1}^k (\psi(i)-\psi(k))\right) = \frac{1}{k}\sum_{i=1}^k \psi(i) - \frac{1}{k}$$
	and hence \eqref{eq:minm2} is equal to
	\begin{align}
		\frac{1}{n}\left(1-\frac{1}{k}\right) + \frac{1}{k}\sum_{i=1}^k \Rel_n(y_1,\ldots,y_{n-1},i).
	\end{align}
	This value is precisely $\Rel_{n-1}(y_1,\ldots,y_{n-1})$, as per Eq.~\eqref{eq:rel_t}, thus verifying \eqref{eq:minm1}. Repeating the argument for $t=n-1$ until $t=0$, we find that 
	$$\Rel_0(\emptyset) = -\En\phi + \left(1-\frac{1}{k}\right) =0,$$
	thus ensuring existence of an algorithm with \eqref{eq:minm0} equal to zero. The other direction of the statement is proved by taking sequences $\by$ uniformly at random from $[k]^n$, concluding the proof.
	
	When $k=2$, the solution $q_t^*$ takes on a simpler form
	$$q_t^*(y_1,\ldots,y_{t-1}) = n[ \En\phi(y_1,\ldots,y_{t-1},-1,\varepsilon_{t+1},\ldots,\varepsilon_n)- \En\phi(y_1,\ldots,y_{t-1},+1,\varepsilon_{t+1},\ldots,\varepsilon_n)],$$
	which is found by equating the two alternatives in \eqref{eq:minm2}.
\end{proof}

\begin{proof}[\textbf{Proof of Lemma~\ref{lem:adaptive_phi_x_y}}]
	As in the proof of Lemma~\ref{lem:cover_multi}, define functions $\Rel_t:(\X \times \{\pm1\})^t\to\reals$ as
	$$\Rel_n(x_{1:n}; y_{1:n}) = -\phi(x_{1:n};y_{1:n})$$
	and
	\begin{align}
		\label{eq:rel_t_adapt}
		\Rel_{t-1}(x_{1:t-1}; y_{1:t-1}) = \En_{\varepsilon_t,x_t} \Rel_{t}(x_{1:t-1},x_t; y_{1:t-1},\varepsilon_t) +  \frac{1}{2n},
	\end{align}
	with $\Rel_{0}(\emptyset)$ being a constant.
 Having observed $x_{1:n-1},y_{1:n-1}$ and $x_n$ at the present time step, we solve
\begin{align}
	\label{eq:lec11:bit_bound_phi_last_with_x}
	\min_{q_n}\max_{y_n}\left\{ \En\left[ \frac{1}{n} \ind{\pred_n\neq y_n} \right] + \Rel_n(x_{1:n},y_{1:n}) \right\}
\end{align}
The same steps as in Lemma~\ref{lem:cover_multi} (for binary prediction) lead to the solution
\begin{align}
	q^*_n(x_{1:n},y_{1:n-1}) = n[\phi(x_{1:n},y_{1:n-1},-1)-\phi(x_{1:n},y_{1:n-1},+1)].
\end{align}
We remark that $q^*_n$ depends on $x_n$, as given by the protocol of the problem. Then \eqref{eq:lec11:bit_bound_phi_last_with_x} equals to
\begin{align}
	\En_{\varepsilon_n}\Rel_n(x_{1:n}; y_{1:n-1},\varepsilon_n) + \frac{1}{2n}
\end{align}
We now take expectation over $x_n$ on both sides:
\begin{align}
	\label{eq:exnminmax}
	\En_{x_n}\min_{q_n}\max_{y_n}\left\{ \En\left[ \frac{1}{n}\sum_{t=1}^n \ind{\pred_t\neq y_t} \right] + \Rel_n(x_{1:n}; y_{1:n}) \right\} 
	=\Rel_{n-1}(x_{1:n-1};y_{1:n-1})
\end{align}
The argument continues back to $t=0$, with
\begin{align}
	\label{eq:rel_phi_t_induct}
	\Rel_t(x_{1:t}; y_{1:t}) = -\En_{x_{t+1:n},\varepsilon_{t+1:n}} \phi(x_{1:n}; y_{1:t},\varepsilon_{t+1:n})+\frac{n-t}{2n}
\end{align}
and
\begin{align}
	\label{eq:lec12:qt_side_info}
	&q^*_t(x_{1:t},y_{1:t-1}) = n\En_{x_{t+1:n},\varepsilon_{t+1:n}}\left[\phi(x_{1:n},y_{1:t-1},-1,\varepsilon_{t+1:n})-\phi(x_{1:n},y_{1:t-1},+1,\varepsilon_{t+1:n})\right].
\end{align}
Finally, 
$$\En\phi(x_{1:n};\varepsilon_{1:n}) = \frac{1}{2}$$
means that $\Rel_0(\emptyset)= 0$.
The algorithm  in \eqref{eq:lec12:qt_side_info} is not implementable: it requires the knowledge of $P_X$. However, all we need is to be able to sample $x_{t+1:n}\sim P_X$ (or have access to unlabeled data), draw independent Rademacher $\varepsilon_{t+1:n}$, and define
\begin{align}
	\label{eq:lec12:random_playout_gen_phi}
	&\widetilde{q}_t^*(x_{1:n},y_{1:t-1},\varepsilon_{t+1:n}) = n\left[\phi(x_{1:n},y_{1:t-1},-1,\varepsilon_{t+1:n})-\phi(x_{1:n},y_{1:t-1},+1,\varepsilon_{t+1:n})\right].
\end{align}
The proof that this strategy yields the same expected mistake bound against adaptive Nature relies on the technique we term random playout. In this case the proof is not difficult. Consider \eqref{eq:lec11:bit_bound_phi_last_with_x} at step $t$ and use the inductive definition of $\Rel_t$ in \eqref{eq:rel_phi_t_induct}:
\begin{align}
	\label{eq:bit_bound_phi_last_with_x_plugin}
	\min_{q_t}\max_{y_t}\left\{ \En\left[ \frac{1}{n} \ind{\pred_t\neq y_t} \right]  -\En_{x_{t+1:n},\varepsilon_{t+1:n}} \phi(x_{1:n}; y_{1:t},\varepsilon_{t+1:n})+\frac{n-t}{2n} \right\}
\end{align}
In the above minimum over $q_t$, let us choose the randomized strategy with the mean given by \eqref{eq:lec12:random_playout_gen_phi}, thus passing to an upper bound of 
\begin{align}
	&\max_{y_t}\left\{ \En_{x_{t+1:n},\varepsilon_{t+1:n}}\En_{\pred_t\sim \widetilde{q}_t^*}\left[ \frac{1}{n} \ind{\pred_t\neq y_t} \right]  -\En_{x_{t+1:n},\varepsilon_{t+1:n}} \phi(x_{1:n}; y_{1:t},\varepsilon_{t+1:n}) \right\} +\frac{n-t}{2n} 
\end{align}
which, in turn, is upper bounded via Jensen's inequality by
\begin{align}
	&\En_{x_{t+1:n},\varepsilon_{t+1:n}} \max_{y_t}\left\{ \En_{\pred_t\sim \widetilde{q}_t^*}\left[ \frac{1}{n} \ind{\pred_t\neq y_t} \right]  - \phi(x_{1:n}; y_{1:t},\varepsilon_{t+1:n})\right\} +\frac{n-t}{2n} .
\end{align}
The choice of $\widetilde{q}_t^*$ makes the two possibilities for $y_t\{\pm1\}$ identical in terms of their value (such a strategy is called an \emph{equalizer}) and is optimal. With routine algebra, the above expression is equal to 
\begin{align}
	&\En_{x_{t+1:n},\varepsilon_{t+1:n}}  \En_{\varepsilon_t} \phi(x_{1:n}; y_{1:t-1},\varepsilon_t,\varepsilon_{t+1:n}) +\frac{n-t+1}{2n}.
\end{align}
Taking expectation with respect to $x_t$ yields $\Rel_{t-1}(x_{1:t-1}; y_{1:t-1})$, completing the recursion step. Moreover, because of this value is optimal, so is the randomized strategy and, hence, the inequality in the above proof is an equality.

Regarding the required stability condition on $\phi$, we see that it is simply that \eqref{eq:lec12:qt_side_info} is within the range $[-1,1]$. In particular, it is implied by the assumed stability condition.

Finally, because of the order of expectation, minimum, and maximum in \eqref{eq:exnminmax}, the choice of $y_t$ for Nature may depend on $x_t$ (which is drawn from an unknown distribution), on $q_t$ (but not on $\pred_t$), and on the history. This ensures that the prediction strategy works against an adaptive Nature.

\end{proof}

\bibliography{paper}

\begin{thebibliography}{SLKS16}

\bibitem[Bar14]{boaz14notes}
B.~Barak.
\newblock Sum of squares upper bounds, lower bounds, and open questions, 2014.
\newblock Lecture notes.

\bibitem[Bla95]{Blackwell95}
D.~Blackwell.
\newblock Minimax vs. bayes prediction.
\newblock {\em Probability in the Engineering and Informational Sciences}, 9:pp
  53--58, 1995.

\bibitem[CBS11]{cesa2011efficient}
N.~Cesa-Bianchi and O.~Shamir.
\newblock Efficient online learning via randomized rounding.
\newblock In {\em Advances in Neural Information Processing Systems}, pages
  343--351, 2011.

\bibitem[Cov65]{Cover65behaviour}
T.~Cover.
\newblock Behaviour of sequential predictors of binary sequences.
\newblock In {\em Proc. 4th Prague Conf. Inform. Theory, Statistical Decision
  Functions, Random Processes}, 1965.

\bibitem[CT12]{chlamtac2012convex}
E.~Chlamtac and M.~Tulsiani.
\newblock Convex relaxations and integrality gaps.
\newblock In {\em Handbook on semidefinite, conic and polynomial optimization},
  pages 139--169. Springer, 2012.

\bibitem[GM12]{gartner2012approximation}
B.~G{\"a}rtner and J.~Matousek.
\newblock {\em Approximation algorithms and semidefinite programming}.
\newblock Springer Science \& Business Media, 2012.

\bibitem[GS13]{guruswami2013rounding}
V.~Guruswami and A.~K. Sinop.
\newblock Rounding {L}asserre {SDP}s using column selection and spectrum-based
  approximation schemes for graph partitioning and {Q}uadratic {IP}s.
\newblock {\em arXiv preprint arXiv:1312.3024}, 2013.

\bibitem[Hag56]{hagelbarger1956seer}
D.W. Hagelbarger.
\newblock Seer, a sequence extrapolating robot.
\newblock {\em Electronic Computers, IRE Transactions on}, pages 1--7, 1956.

\bibitem[KT02]{kleinberg2002approximation}
J.~Kleinberg and E.~Tardos.
\newblock Approximation algorithms for classification problems with pairwise
  relationships: Metric labeling and markov random fields.
\newblock {\em Journal of the ACM (JACM)}, 49(5):616--639, 2002.

\bibitem[Rot13]{rothvoss2013lasserre}
T.~Rothvo{\ss}.
\newblock The lasserre hierarchy in approximation algorithms.
\newblock {\em Lecture Notes for the MAPSP}, pages 1--25, 2013.

\bibitem[RS15]{RakSri15hierarchies}
A.~Rakhlin and K.~Sridharan.
\newblock Hierarchies of relaxations for online prediction problems with
  evolving constraints.
\newblock In {\em COLT}, 2015.

\bibitem[RS16]{rakhlin2016bistro}
A.~Rakhlin and K.~Sridharan.
\newblock {BISTRO}: An efficient relaxation-based method for contextual
  bandits.
\newblock In {\em International Conference on Machine Learning}, 2016.

\bibitem[RSS12]{rakhlin2012relax}
A.~Rakhlin, O.~Shamir, and K.~Sridharan.
\newblock Relax and randomize: From value to algorithms.
\newblock In {\em Advances in Neural Information Processing Systems 25}, pages
  2150--2158, 2012.

\bibitem[RST15]{rakhlin2015sequential}
A.~Rakhlin, K.~Sridharan, and A.~Tewari.
\newblock Sequential complexities and uniform martingale laws of large numbers.
\newblock {\em Probability Theory and Related Fields}, 161(1-2):111--153, 2015.

\bibitem[Sha53]{shannon1953mind}
C.E. Shannon.
\newblock A mind-reading machine.
\newblock {\em Bell Laboratories memorandum}, 1953.

\bibitem[SLKS16]{SyrgkanisLKS16}
V.~Syrgkanis, H.~Luo, A.~Krishnamurthy, and R.~E. Schapire.
\newblock Improved regret bounds for oracle-based adversarial contextual
  bandits.
\newblock {\em CoRR}, abs/1606.00313, 2016.

\end{thebibliography}
\bibliographystyle{alpha}

\end{document}